\documentclass[conference]{IEEEtran}
\IEEEoverridecommandlockouts
\usepackage{multirow}
\usepackage{array}
\usepackage{makecell}
\usepackage{marvosym}
\usepackage{cite}

\usepackage{amsmath,amssymb}
\usepackage{amsthm}
\usepackage{balance}

\usepackage{graphicx}
\graphicspath{{figures/}}
\usepackage{subcaption}
\usepackage[export]{adjustbox}
\usepackage{flushend}

\usepackage{listings}

\usepackage{color}
\usepackage{soul}
\usepackage{tikz}
\newcommand*{\circled}[1]{\lower.7ex\hbox{\tikz\draw (0pt, 0pt)%
    circle (.5em) node {\makebox[1em][c]{\small #1}};}}

\usepackage{algorithm}
\usepackage{algpseudocode}

\theoremstyle{definition}
\newtheorem{definition}{Definition}
\newtheorem{theorem}{Theorem}

\newtheorem{lemma}{Lemma}

\newcolumntype{I}{!{\vrule width 1.2pt}}
\newlength\savedwidth

\newlength\savewidth
\newcommand\shline{\noalign{\global\savewidth\arrayrulewidth
                           \global\arrayrulewidth 1.2pt}%
                  \hline
                  \noalign{\global\arrayrulewidth\savewidth}}

\usepackage{CJKutf8}

\usepackage{changes}

\begin{document}
\begin{CJK*}{UTF8}{gbsn}

\title{Fast Exact NPN Classification with Influence-aided Canonical Form
\thanks{\IEEEauthorrefmark{1}These authors contributed equally to this work; $^\text{\Letter}$Corresponding authors.}
}

\author{\IEEEauthorblockN{
Yonghe~Zhang$^{1,}$\IEEEauthorrefmark{1},
Liwei~Ni$^{2,4,}$\IEEEauthorrefmark{1},
Jiaxi~Zhang$^{3,\text{\Letter}}$,
Guojie~Luo$^{3}$,
Huawei~Li$^{4,5}$
and Shenggen~Zheng$^{2,\text{\Letter}}$,
}
\IEEEauthorblockA{
$^1$Shenzhen University, Shenzhen, China}
\IEEEauthorblockA{
$^2$Peng Cheng Laboratory, Shenzhen, China}
\IEEEauthorblockA{
$^3$School of Computer Science, Peking University, Beijing, China}
\IEEEauthorblockA{
$^4$University of Chinese Academy of Sciences, Beijing, China}
\IEEEauthorblockA{
$^5$Institute of Computing Technology, Chinese Academy of Sciences, Beijing, China}
\IEEEauthorblockA{
Emails: zhangjiaxi@pku.edu.cn, 
zhengshg@pcl.ac.cn
}}

\maketitle


\begin{abstract}

NPN classification has many applications in the synthesis and verification of digital circuits.
The canonical-form-based method is the most common approach, designing a canonical form as representative for the NPN equivalence class first and then computing the transformation function according to the canonical form.
Most works use variable symmetries and several signatures, mainly based on the cofactor, to simplify the canonical form construction and computation.
This paper describes a novel canonical form and its computation algorithm by introducing Boolean influence to NPN classification,  which is a basic concept in analysis of Boolean functions. 
We show that influence is input-negation-independent, input-permutation-dependent, and has other structural information than previous signatures for NPN classification. 
Therefore, it is a significant ingredient in speeding up NPN classification.
Experimental results prove that influence plays an important role in reducing the transformation enumeration in computing the canonical form.
Compared with the state-of-the-art algorithm implemented in \textit{ABC}, our influence-aided canonical form for exact NPN classification gains up to 5.5x speedup.

\end{abstract}


\section{Introduction}
Boolean function \textbf{NPN}~(Negation-Permutation-Negation) classification groups functions into NPN equivalent classes.
Each function in the same equivalent class can be obtained from the other by three transformations: \textbf{N}egating the inputs, \textbf{P}ermuting the inputs, or \textbf{N}egating the output.
NPN classification is one of the critical steps in many applications, such as logic synthesis~\cite{yang2012lazy,haaswijk2017classifying}, technology mapping~\cite{cong2001boolean}, and verification~\cite{mohnke2001application}.
The speed of classification will directly affect the efficiency of these tools because of the exploding search space of the Boolean function.

Due to its great importance, NPN classification has been well studied in the past decades.
The canonical-form-based approach is one of the most commonly used methods for exact NPN classification.
First, it designs a complete and unique representation for a Boolean function called canonical form.
This canonical form is taken as the representative, and 
two functions fall into the same NPN equivalent class when their canonical forms are the same.
Then a computation algorithm decides the output polarity, the phase assignment, and the input order of each variable for a function to get the minimum canonical form.

Truth table is a basic canonical form, and the exhaustive transformation enumeration is the most primitive computation method.
However, $2^{n+1}n!$ transformations ($2^n$ for input negation, $n!$ for input permutation, $2$ for output negation) should be enumerated for an $n$ input function, which is impractical when $n$ is large.
Many optimization methods have been proposed to improve the NPN classification efficiency, including constructing new canonical forms~\cite{hinsberger1998boolean,debnath2004efficient,abdollahi2005new,agosta2009transform} and designing better computation methods~\cite{chai2006building,kravets2000generalized,huang2013fast,abdollahi2008symmetry,petkovska2016fast,zhou2019fast,zhou2020fast}.
Pruning transformation enumeration is the key to reducing computation time in the canonical-form-based approach.
Various works exploited several signatures~\cite{hinsberger1998boolean, chai2006building, zhou2019fast, zhou2020fast} and variable symmetry~\cite{chai2006building,kravets2000generalized,huang2013fast,abdollahi2008symmetry,petkovska2016fast,zhou2019fast,zhou2020fast} to prune the enumeration space.
Signatures are compact representations that characterize a Boolean function's properties. 
However, these signatures also need to be computed, introducing additional runtime.
A co-design of canonical form and computation algorithm~\cite{zhou2020fast} is required to achieve better speedup.

In Boolean functions analysis, the concept of \textit{influence} is defined to measure the probability that one input variable affects the value of the function~\cite{kahn1989influence,o2014analysis}. 
The value of influence of a variable $x_i$ can be computed easily from  Boolean difference \cite{akers1959theory}. 
This concept has been applied to hardware security~\cite{waksman2013fanci,zhang2013veritrust} since inputs with high influence are of more interest in critical applications.
Existing work has proved that Boolean influence can be regarded as a type of signature for Boolean matching~\cite{zhang2023rethinking}.
But this work does not provide an exact matching solution for the functions in the same NPN classes, posing challenges for its practical application.
This paper explores Boolean influence's great potential in canonical-form-based exact NPN classification.

The main contributions are summarized as the following: 
\begin{itemize}
\item
We redefine the concept of Boolean influence and highlight that computing the redefined influence using commonly utilized cofactor signatures is straightforward.
Consequently, incorporating influence into canonical form-based method requires only a modest amount of additional computation.
\item 
We analyze some properties of influence in NPN classification scenarios, including that it is phase-independent and permutation-dependent. 
Such properties can help reduce the enumeration times of permutation in canonical form computation.
\item We devise two influence-aided canonical forms based on the existing canonical-form-based NPN classification method.
One directly introduces Boolean influence into the canonical form, while the other replaces some signatures of high computation complexity with influence.
\item The experiments demonstrate that both proposed two influence-aided canonical forms in this paper effectively reduce the number of final transformation enumerations and achieve significant speedup compared to the state-of-the-art method.
\end{itemize}


\section{Background}\label{sec:background}
In this section, we first introduce several commonly used signatures and techniques in NPN classification. 
Subsequently, we illustrate a state-of-the-art canonical form-based classification method. 

\subsection{Frequently Used Signatures}
\label{sec:signature}
We first list some basic notations before introducing the definitions of signatures.
An $n$-input and single-output Boolean function, $f(X):B^{n}\to B$, $B$=$\{0,1\}$, where $X$=$(x_{1},x_{2},...,x_{n})$,$x_{i}\in B$,$1\leq i \leq n$ is a binary vector with size $n$, representing the input variables. 
$X$ can be considered as a binary number $(x_{n}x_{n-1}...x_{1})_{2}$, whose value is $m$, denoted as $X_{(m)}$. 
Besides, we denote $X^i$ as negating the $i$-th variable in $X$.
The truth table $T(f)=(f(X_{2^n-1}),...,f(X_{(1)}),f(X_{(0)}))$ is a bit vector with size $2^n$.

The \emph{satisfy count} of $f$ is the number of all minterms for which $f$ evaluates to 1, denoted as $|f|$, where $|f|=\sum_{X\in B^n} f(X)$.
A \emph{cofactor} of $f$ with respect to a literal $x_i$($\overline{x_i}$) is the function derived by setting $x_i$($\overline{x_i}$) to 1 and denoted by $f_{x_i}$($f_{\overline{x_i}}$).
The 0th-order cofactor signature is the satisfy count, and the 1st-order cofactor signature is the satisfy counts of the cofactors with respect to each variable, denoted as $S_{cof}(f)=\{|f|,|f_{x_1}|,...,|f_{x_n}|\}$.
Moreover, higher-order cofactor signature is the satisfy counts of the cofactors with respect to multiple variables~\cite{abdollahi2008symmetry}. 

Boolean functions can be represented as the disjunction of minterms. 
In this representation, \emph{row sums}~(RS) is a sorted vector composed of the number of positive variables in each minterm.
A single value can be computed by adding the square of each value in row sums, called \emph{sum of squared row sums}~(SSRS).
Another efficient value \emph{sum of exponential row sums}~(SERS) can be obtained from RS~\cite{agosta2009transform,zhou2020fast}.
The SERS of a Boolean function is the 0th-order \emph{shifted-cofactor} signature, denoted as $S^0_{scc}(f)=\|f\|$;
the 1st-order shifted-cofactor signatures are SERS of the cofactors with respect to each variable, denoted as $S^1_{scc}(f)=\{\|f_{x_1}\|,...,\|f_{x_n}\|\}$.

\subsection{Variable Symmetries}
\label{sec:symmetry}

Two variables $x_i$ and $x_j$ are \textit{symmetric} in $f$ if the function value does not change when $x_i$ and $x_j$ are swapped.
A \textit{symmetric group} is composed of symmetric variables.
The higher-order symmetry is defined between two or more symmetry groups if they can be swapped without changing the function value.
In some cases, one variable in a symmetry group is not symmetric with each variable in another, but these two symmetry groups are higher-order symmetric.

Symmetry variables and higher-order symmetry groups can be permuted and negated like a single variable during the canonical form computation, thus significantly reducing the number of transformation enumerations.
Exploiting variable symmetries and higher-order symmetries are adopted in many NPN classification methods~\cite{abdollahi2005new,abdollahi2008symmetry, petkovska2016fast, huang2013fast, zhou2019fast, zhou2020fast}. 

\subsection{Canonical-Form-based Method}
\label{sec:canonicalformexample}

The key of the canonical-form-based method lies in determining the output polarity, phase assignment, and input order of a Boolean function to get the canonical form of its equivalent class.
The output polarity is associated with output negation, while the input phase assignment corresponds to input negation. 
Furthermore, the input order corresponds to the permutation of inputs.

The transformation space is vast~($2^{n+1}n!$), and signatures can help effectively prune the space.
If a signature fulfills the following requirements, it can help to determine partially the three types of transformations.
The following three requirements can be met: 
1)\emph{Phase requirement}.
The signature value of a single variable is determined solely by its own phase assignment and remains independent of the phase assignment of other variables.
2)\emph{Permutation requirement}.
Applying a permutation to the variables leads to the same permutation of the signature values for each variable.
3)\emph{Polarity requirement}. 
The signature value of the function is primarily determined by the polarity of the output and remains independent of the permutation and phase assignment of input variables.

Cofactor signatures fulfill the three requirements.
They can be used to determine the output polarity and partially decide the input order and the phase of input variables.
This means that cofactor signatures can reduce part of negation and permutation enumeration.
In addition, the computational cost of the cofactor signatures is low.
Therefore, many works used cofactor signatures to build their NPN classification methods~\cite{chai2006building,abdollahi2008symmetry,huang2013fast,zhou2019fast}.
Since the order of the variables does not affect the number of positive variables in the minterm, RS and SSRS are permutation-independent, which means these two values do not fulfill the polarity requirement and permutation requirement.
Therefore, these two signatures cannot distinguish polarity and permutation.
Some works used these two signatures to determine phase assignments of input variables~\cite{chai2006building,zhou2019fast}.

Pruning permutation is more critical because permutation enumeration space~($n!$) is more significant than phase assignment enumeration~($2^n$) when $n$ is larger than 4.
Shifted-cofactor signatures fulfill the permutation requirement and have proved more effective than cofactor in distinguishing the permutation~\cite{agosta2009transform}.
Thus, the integration of cofactor and shifted cofactor in the canonical form~\cite{zhou2020fast} outperforms the utilization of cofactor alone with SSRS in terms of efficacy~\cite{zhou2019fast}.
Furthermore, the state-of-the-art work~\cite{zhou2020fast} introduces cost signature $C_{p}$ to assess the overhead of permutations under different phase assignments, aiming to select more promising phase assignments for permutation enumeration and ultimately achieve better speedup.

\begin{figure}[tbp]
\setlength{\belowcaptionskip}{0pt}
\centering
\begin{subfigure}[b]{0.24\textwidth}
\centering
\includegraphics[scale=0.45]{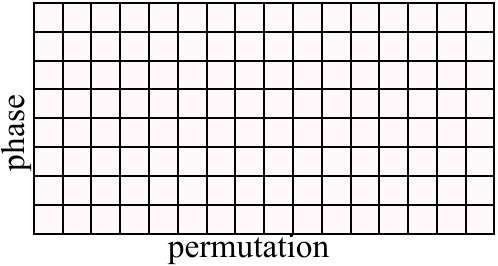}
\vspace{-0.3em}
\caption{\footnotesize{Original enumeration space.}}
\label{fig:original}
\end{subfigure}
\begin{subfigure}[b]{0.24\textwidth}
\centering
\includegraphics[scale=0.45]{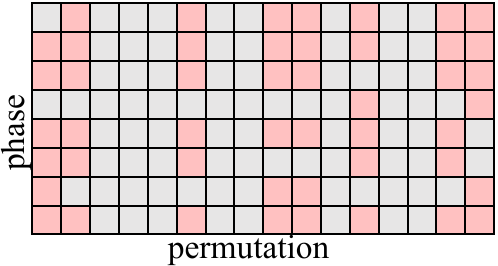}
\vspace{-0.3em}
\caption{\footnotesize{After Variable grouping.}}
\label{fig:prune2}
\end{subfigure}
\begin{subfigure}[b]{0.24\textwidth}
\centering
\includegraphics[scale=0.45]{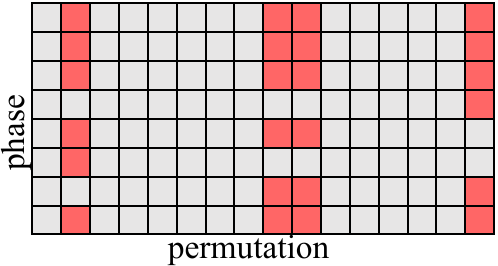}
\vspace{-0.3em}
\caption{\footnotesize{After symmetry detection.}}
\label{fig:prune3}
\end{subfigure}
\begin{subfigure}[b]{0.24\textwidth}
\centering
\includegraphics[scale=0.45]{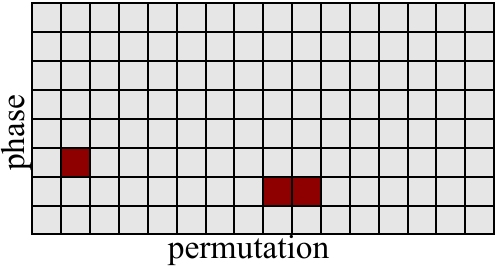}
\vspace{-0.3em}
\caption{\footnotesize{After cost estimation.}}
\label{fig:prune4}
\end{subfigure}
\vspace{-0.5em}
\caption{Schematic of the remaining enumerations at different pruning stages, grey points are pruned enumerations.
(a) All phase assignments and permutations are candidates to obtain the canonical form at the beginning.
(b) After variable grouping, candidate enumerations are separated into different groups.  
(c) Variable symmetry detection further prunes the candidate enumerations.
(d) After cost estimation, only three transformations need to be enumerated to finish the final exhaustive enumeration.}
\label{fig:pruningdemo}
\end{figure}

Figure~\ref{fig:pruningdemo} shows the schematic of the remaining enumeration space at different pruning stages of the cost-aware hybrid signature canonical form proposed in previous work~\cite{zhou2020fast}.
First, they decide the output polarity and group variables by cofactor signatures.
Certain enumerations of phase assignments can be pruned through this step.
Subsequently, symmetry detection within each group will be conducted to eliminate certain phase assignments and permutations.
Next, for the variables that have yet to determine their phase assignments and permutations, the cost signature is utilized to identify phase assignments with lower permutation costs. 
Then, the shifted-cofactor signature is employed to further prune permutations.
After all pruning strategies, an exhaustive enumeration of residual transformations is used to achieve exact NPN classification.
However, the effectiveness of pruning permutations using the shifted-cofactor signature remains limited. 
This paper discovers a more effective signature to prune permutations, Boolean influence.



\section{The Power of Influence}\label{sec:influence}

In this section,  we provide an in-depth examination of Boolean influence in the context of NPN classification, including its redefined definition, facile computation, and some properties. For primitive definitions and proofs of difference and influence,  one may refer to \cite{o2014analysis,akers1959theory}. 

\subsection{Redefined Influence and Facile Computation}

In canonical form-based NPN classification methods, some signatures can be used to prune the search space of transformation enumerations.
But the computation of these signatures cannot be too complicated.
Otherwise, the additional computational overhead will diminish the benefit.
To better explain the computation of Boolean influence, we redefine it using difference~\cite{akers1959theory}.


\begin{definition} 
\label{def: differenc}
(\emph{difference})   \cite{akers1959theory}.
Given a Boolean function $f(X)$, the Boolean difference with respect to a basic variable $x_i$ is defined as follows:
$$\frac{\delta f}{\delta x_i}=f(x_1,\cdots,x_i,\cdots,x_n)\oplus  f(x_1,\cdots,\overline{x_i},\cdots,x_n).$$
\end{definition}

\begin{definition}
\label{def:inf}
(\emph{redefined influence}).
The \textit{influence} of the variable $x_{i}$ on $f$ is defined as follows:  
$$inf_i(f)=\sum\limits_{X \in B^n} \frac{\delta f}{\delta x_i}(X).$$
\end{definition}



We drop a global coefficient of $1/2^n$ in the above definition compared to the original definition in~\cite{o2014analysis}. 
The influence of a variable $x_i$ is the number of functions that flipping the variable $x_i$ flips the value of the function.  
In some sense, one can see the influence of a variable is a measure of the importance of the variable in all input variables.  
From the above definition, influence can be computed for each input variable.   
It has the potential to be a signature of a Boolean function to distinguish variables for NPN classification.

The 1st-order cofactor is easy to compute, and many canonical-form-based methods have adopted it.
The characteristics of cofactor and influence seem to be quite different from each other.  
However, they are relevant to each other in some respect. 
We will show some detailed relationships between these concepts as followings. 
We can get the following Lemma from previous definitions  \cite{akers1959theory}.

\begin{lemma}
\label{lm:difference}
(\emph{difference vs. cofactor}).
Difference can be computed by cofactor as follows: 
$$\frac{\delta f}{\delta x_i}(X)= f_{x_{i}}(X) \oplus f_{\overline{x_{i}}}(X).$$
\end{lemma}

From Definition~\ref{def:inf} and Lemma~\ref{lm:difference}, we can easily derive a facile computation method for the redefined influence.

\begin{lemma}
\label{lm:infcompute}
(\emph{influence computation}). The influence can be obtained from the cofactor signature as follows:
$$inf_{i}(f)=\sum\limits_{X \in B^n} \frac{\delta f}{\delta x_i}(X)=\sum\limits_{X \in B^n} f_{x_i}(X)\oplus f_{\overline {x_i}}(X).$$
\end{lemma}



Lemma~\ref{lm:infcompute} demonstrates that the influence can be intuitively computed using cofactors.
This also implies that if we define a canonical form using influence, the additional computation overhead is low.
The above two lemmas also show that influence further explores cofactors.
It explores the difference between cofactor signature pairs and has extra structural information. 
An $n$-input Boolean function can be seen as an $n$-dimensional hypercube.  
Cofactor focuses on the statistical characteristics of a face $x_i$ of the hypercube, while influence cares about the difference between two opposite faces $x_i$ and $\overline {x_i}$ of the hypercube.  
These two features are complementary in some sense.

\subsection{Understanding Boolean Influence in NPN Classification}
\label{sec:infproperty}

To prune transformation enumerations by incorporating influence into the canonical form, it is necessary to first understand the relationship between influence signatures and transformations.
It can be seen from the following theorem that the influence signature changes only with permutation~(P) transformation, and any input phase assignment and polarity assignment will not change the influence.

\begin{theorem}
(\emph{phase-independent}).
Variable negation does not change Boolean influence.
\end{theorem}

\begin{proof}
Supposing $g$ is derived from $f$ through the phase assignment of any input $i$, $g(X)=f(X^i)$.
According to definition~\ref{def: differenc} and definition~\ref{def:inf}, we have  $\frac{\delta f}{\delta x_i}(X)  =   f_{x_i}(X)  \oplus f_{\overline {x_i}}(X)  =g_{\overline {x_i}}(X)  \oplus g_{x_i}(X)   =   \frac{\delta g}{\delta x_i}(X) $.
Thus, $inf_{i}(g)={\Sigma}_{X \in B^n} \frac{\delta g}{\delta x_i}(X) = {\Sigma}_{X \in B^n}\frac{\delta f}{\delta x_i}(X)  = inf_{i}(f)$.
\end{proof}

\begin{theorem}
(\emph{polarity-independent}).
Output negation does not change Boolean influence.
\end{theorem}

\begin{proof}
Supposing $g$ is derived from $f$ through the polarity assignment, $g(X)$=$\overline{f(X)}$.
According to definition~\ref{def: differenc}  and definition~\ref{def:inf}, we have $\frac{\delta g}{\delta x_i}(X)$=$g_{x_i}(X) \oplus g_{\overline x_i}(X)$=$\overline{f_{x_i}(X)} \oplus \overline{ f_{\overline {x_i}}(X) }$ = $f_{x_i}(X) \oplus f_{\overline {x_i}}(X)$=$\frac{\delta f}{\delta x_i}(X) $.
Thus, $inf_{i}(g)$ =$inf_{i}(f)$.
\end{proof}

\begin{theorem}
(\emph{permutation-dependent}).
The influence of variable changes synchronously  with the permutation transformation.
\end{theorem}

\begin{proof}
Assuming that $g(X) = f(\pi(X))$, where $\pi$ represents a permutation transformation and $k = \pi(i)$,  we observe the following:
 $g(X)=f(\pi(X))$
and 
    $g(X^k)=f(\pi(X^k)) =f(\pi(X)^i)$.

According to definition~\ref{def: differenc} and definition~\ref{def:inf}, we have
\begin{equation}
\nonumber
    \begin{aligned}
    inf_{k}(g) &= \sum\limits_{X \in B^n} \frac{\delta g}{\delta x_k}(X)=\sum\limits_{X \in B^n} g(X) \oplus g(X^k)\\
    &= \sum_{X \in B^n} f(\pi(X)) \oplus f(\pi(X)^i).
    \end{aligned}
\end{equation}
An essential observation is that the set $\{X|X \in B^n\} = \{\pi(X)|X \in B^n\}$, which leads us to:
\begin{equation}
\nonumber
    \begin{aligned}
    inf_{k}(g) &= \sum_{X \in B^n} f(\pi(X)) \oplus f(\pi(X)^i)\\
    &= \sum_{X \in B^n} f(X) \oplus f(X^i)=inf_{i}(f) .
    \end{aligned}
\end{equation}
\end{proof}
From the above three theorems, it can be found that influence signature can determine the permutation, which is very helpful for pruning permutation enumeration.

\subsection{An Example of Influence}
\label{sec:infexample}

In order to reduce permutations, one common way used in previous work~\cite{zhou2019fast, zhou2020fast} is to group variables according to signatures. 
Variable grouping can separate the permutation space into isolated ones, and permutations only appear in the same group.
For a 6-bit Boolean function:
\begin{align}
\nonumber
    f=&\overline{x_3}\overline{x_5}x_6+\overline{x_3}x_4x_6+x_3\overline{x_5}\overline{x_6}+x_3x_4\overline{x_6}+\\
    &x_2x_3\overline{x_4}x_5+x_1\overline{x_3}\overline{x_4}x_5\\
    =&(\mathrm{5DAE51AE5DA251A2})_{16}
\end{align}
The cofactor signatures, shifted-cofactor signatures, and influence signatures of each variable are listed in Table~\ref{tab:signatureexample}.
And Figure~\ref{fig:5DAE51AE5DA251A2} also shows the logic graph of $f$ with the support set of inputs in And-Inverter Graph format.
If the signatures~(combination) of two variables are the same, then this signature~(combination) can not distinguish these two variables.
These two variables belong to the same group using such signature~(combination).
Table~\ref{tab:groupexample} shows the variable grouping results using the different signature combinations.
For this Boolean function, the combination of cofactor and influence signature splits two more groups than the combination of cofactor and shifted-cofactor signature.
This shows that the influence signature prunes more permutation enumeration than the shifted-cofactor signature.
This also shows that the influence signature has great potential for permutation pruning in the canonical form-based NPN classification method.

\begin{figure}[htbp]
    \centering
    \small
    \includegraphics[scale=0.25]{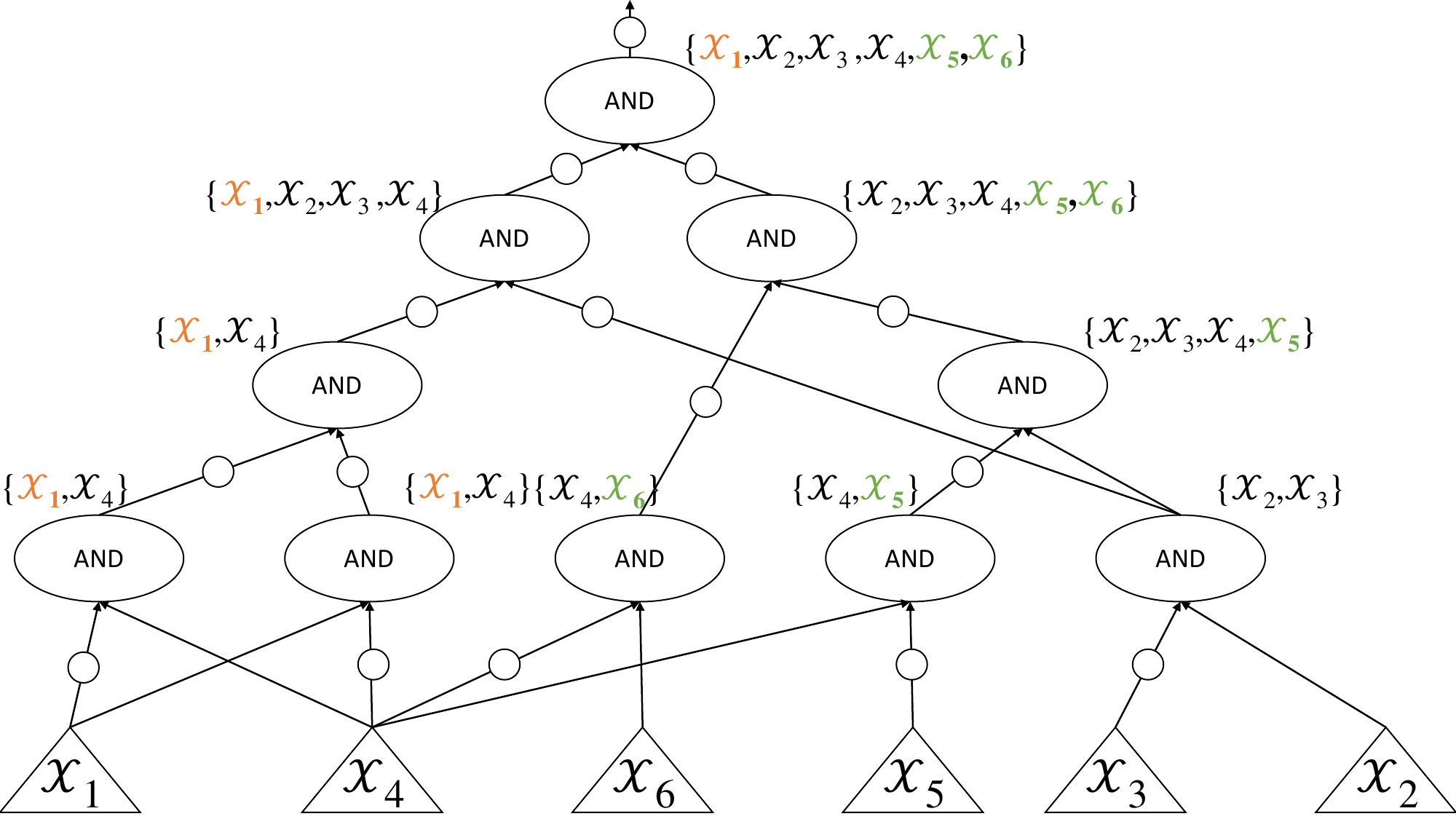}
    \caption{AND-Inverter graph for $f$=$(\mathrm{5DAE51AE5DA251A2})_{16}$}
    \label{fig:5DAE51AE5DA251A2}
\end{figure}

\vspace{5pt}
\begin{table}[htbp]
\caption{Signatures results.}
\label{tab:signatureexample}
    \centering
    \begin{tabular}{c|c|c|c|c|c|c}
        \hline
        \textbf{Signatures}       & $x_1$ & $x_2$ & $x_3$ & $x_4$ & $x_5$ & $x_6$ \\
        \hline
        cofactor       & 16    & 16    & 16    & 16    & 18    & 18  \\
        shifted-cofactor       & 81    & 135   & 81    & 135   & 159   & 159 \\
        influence       & 24    & 8     & 28    & 4     & 4     & 4  \\
        \hline
    \end{tabular}
\end{table}

\vspace{5pt}
\begin{table}[htbp]
\caption{Variable grouping results.}
\label{tab:groupexample}
    \centering
    \begin{tabular}{c|c}
        \hline
        \textbf{Signatures~(Combination)}       &  \textbf{Variable Groups} \\
        \hline
        cofactor              & $(x_1, x_2, x_3, x_4)$, $(x_5,x_6)$  \\
        cofactor+shifted-cofactor          & $(x_1, x_3)$, $(x_2, x_4)$, $(x_5, x_6)$ \\
        cofactor+influence         & $(x_1)$, $(x_2)$, $(x_3)$, $(x_4)$, $(x_5, x_6)$ \\
        All three signatures      & $(x_1)$, $(x_2)$, $(x_3)$, $(x_4)$, $(x_5, x_6)$ \\
        \hline
    \end{tabular}
\end{table}

\section{Hybrid Influence-aided Canonical Form}
\label{sec:canonical}

In the previous section, we pointed out that influence signature is permutation-dependent and fulfills the permutation requirement, which means influence has excellent potential for pruning permutation enumerations.
In this section, we introduce the hybrid influence-aided canonical form and its computation algorithm.

\subsection{Canonical Form Definition}
\label{sec:influencecanonical}

\begin{definition}
(\emph{hybrid influence-aided signature vector}). 
The \emph{hybrid influence-aided signature vector} $S_{hi}(f)$ of a Boolean function $f$ is denoted as a vector composed of cofactor signatures $S_{cof}(f)$, influence signatures $S_{inf}(f)$, the permutation cost signature $C_{p}(f)$, the 0th-order and 1st-order shifted-cofactor signatures  $S^0_{scc}(f)$ and $S^1_{scc}(f)$, followed by the truth table $T(f)$, that is,
\begin{equation}
\nonumber
\begin{aligned}
    S_{hi}(f) =&S_{cof}(f)  S_{inf}(f) C_{p}(f)  S^0_{scc}(f) S^1_{scc}(f) T(f) \\
     =&\Big( \lvert f \rvert, |f_{x_{1}}|, |f_{x_{2}}|,..., |f_{x_{n}}|, \\  
      &inf_{x_1}(f),inf_{x_2}(f),...,inf_{x_n}(f),  \\
      &C_{p}(f), \|f\|, \|f_{x_{1}}\|, \|f_{x_{2}}\|,...,   \|f_{x_{n}}\|, \\
      &f(X_{(2^n-1)}), ..., f(X_{(1)}),  f(X_{(0)} ) \Big),
\end{aligned}
\end{equation}
where the hybrid signatures are concatenated as a larger vector.

\end{definition}

\begin{theorem}
\label{tm:hybridinfvec}
(\emph{The uniqueness of $S_{hi}(f)$}).
For two Boolean functions $f$ and $g$: $f \neq g$ if and only if $S_{hi}(f) \neq S_{hi}(g)$.
\end{theorem}

\begin{proof}
(1)Sufficiency:
If $f \neq g$, then $T(f) \neq T(g)$.
Thus, $S_{hi}(f) \neq S_{hi}(g)$.

(2)Necessity:
If $S_{hi}(f) \neq S_{hi}(g)$, then there is at least one part $s \in \{S_{cof}, S_{inf}, C_{p}, S_{ scc}, T\}$ s.t. $s(f) \neq s(g)$.
If $s$ = $S_{inf}$, according to the definition of $S_{inf}$, there is at least one variable $x_i$ s.t. $inf_{i}(f) \neq inf_{i}(g)$.
However, if $f$=$g$, then $inf_{x_i}(f) = inf_{x_i}(g)$.
By contradiction, it is evident that $f \neq g$.
If $s \neq S_{inf}$, then $s$ must belongs to $\{ S_{cof}, C_{p}, S_{scc}, T \}$.
According to the existing works~\cite{agosta2009transform, zhou2019fast, zhou2020fast}, regardless of which part $s$ belongs to, there must be $g \neq f$.
Therefore, Theorem~\ref{tm:hybridinfvec} holds.
\end{proof}

Theorem~\ref{tm:hybridinfvec} shows that $S_{hi}(f)$ can uniquely and completely represent $f$, and a canonical form can be derived from this vector.

\begin{definition}
(\emph{hybrid influence-aided canonical form}). The \emph{hybrid influence-aided canonical form} of a Boolean function $f$, denoted by $\kappa_{hi}(f)$,  is defined as the function in $[f]$ with the minimum $S_{hi}(f)$, i.e. $\kappa_{hi}(f)=min_{S_{hi}}([f])$, where $[f]$ represents the NPN equivalent class of $f$.
\end{definition}

\begin{theorem}
\label{tm:hybridinfcf}
(\emph{$\kappa_{hi}(f)$ is an NPN canonical form}.)
$\kappa_{hi}(g)=\kappa_{hi}(f)$ if and only if $g\in [f]$.
That is, $f$ and $g$ belong to the same NPN equivalent class.
\end{theorem}

\begin{proof}
Let $h$=$\kappa_{hi}(f)$=$argmin(S_{hi}(k)|k\in [f])$, and $h$=$\pi_{f\to h}(f)$.
Then, we have $h \in [f]$.

(1)Sufficiency: $h \in [g]$ due to $h$=$\kappa_{hi}(g)$, thus $h$ can be obtained from $g$ after NPN transformations;
$h \in [f]$ due to $h$=$\kappa_{hi}(f)$, thus $f$ can be obtained from $h$ after NPN transformations;
Thus, $f$ can be obtained from $g$ after NPN transformations due to the transitivity of NPN equivalence, that is, $g\in [f]$.

(2)Necessity:
If $g\in [f]$, then $[g]$=$[f]$.
Thus, $\kappa_{hi}(g)$=$min_{S_{hi}}( [g])$=$min_{S_{hi}}( [f])$=$\kappa_{hi}(f)$.

Therefore, Theorem~\ref{tm:hybridinfcf} holds.
\end{proof}
Since $S_{hi}(f)$ can uniquely and completely represent $f$, the equation $\kappa_{hi}(f)=min_{S_{hi}}( [f])$ has a only unique value.
Theorem~\ref{tm:hybridinfcf} guarantees that each Boolean function has a unique \emph{hybrid influence-aided canonical form}.

\subsection{Canonical Form Computation}
\label{sec:computation}

The proposed algorithm in Algorithm~\ref{algo:compute} computes $\kappa_{hi}(f)$ for a given function $f$.
The basic idea of the algorithm is to prune the undetermined phase assignment and permutation transformations based on $S_{hi}(f)$ and then performing exhaustive enumeration to decide the undetermined transformations.
After exhaustive enumeration, we can get $\kappa_{hi}(f)$.
Supposing that the total number of Boolean functions is $m$ and each function has $n$ bits. 
The size of truth table for each function is $2^n$. 
The time complexity to compute cofactor is $O(mn2^n)$. 
According to the computational method that we give in Lemma 1 and 2, the time complexity to compute influence is also $O(mn2^n)$. 
Therefore, the complexity of Algorithm~\ref{algo:compute}  is the same of the one in \cite{zhou2020fast}.

\newcommand{\mycomment}[1]{
\Comment{{\color{lightgray}#1}}
}
\vspace{5pt}
\begin{algorithm}[ht] 
\caption{Computing the Hybrid Influence-aided Canonical Form}
\label{algo:compute} 
\begin{algorithmic}[1] 
\Require Boolean function $f$ with $n$ input variables
\Ensure Canonical form $\kappa_{hi}(f)$

\State Initialize $G$, $U_{phase}$ and $U_{perm}$
\State Decide output polarity by $S_{cof}(f)$ \mycomment{Pruning by $S_{cof}(f)$}
\State Update $G$ and $U_{phase}$ by $S_{cof}(f)$
\State Update $U_{perm}$ according to $G$
\For{group in $G$}
    \State Detect symmetry
    \State Update $U_{phase}$ and $U_{perm}$
\EndFor
\State Update $G$ using $S_{inf}(f)$ \mycomment{Pruning by $S_{inf}(f)$}
\State Update $U_{perm}$ according to $G$
\State Generate all phase assignment $A_{phase}$ based on $U_{phase}$
\For{phase in $A_{phase}$} \mycomment{Pruning by $S_{scc}(f)$}
    \State compute $C_{p}$ according to $S_{scc}^0(f)$ and $S_{scc}^1(f)$
    \If{$C_{p}$ is minimum}
        \State Refresh $C_{phase}$
    \EndIf
\EndFor
\For{candidate in $C_{phase}$} \mycomment{Exhaustive enumeration}
    \State Exhaustive enumeration according to $U_{perm}$
    \State Record $f_{best}$ with the minimum truth table $T(f)$
\EndFor
\State \Return $\kappa_{hi}(f)$=$f_{best}$.
  
\end{algorithmic}
\end{algorithm}

Firstly, we initialize three empty sets to preserve groups~($G$), phase-undetermined variables ~($U_{phase}$), and permutation-undetermined variables~($U_{perm}$).
The output polarity can be determined by $|f|$~(line 2).
The $S_{cof}(f)$ can help group variables and update $U_{phase}$~(line 3).
Then we update $U_{perm}$ based on the rule that variables in different groups cannot be permuted~(line 4).
Next, symmetry variables are detected in all groups to reduce variables in $U_{phase}$ and $U_{perm}$~(line 5 to line 8).
The above steps prune enumerations mainly based on cofactor signatures and variable symmetries, which have been extensively studied by previous works~\cite{abdollahi2008symmetry,huang2013fast,zhou2019fast}.
After pruning by $S_{cof}(f)$, we update $G$ and $U_{perm}$ using $S_{inf}(f)$~(line 9 to line 10).
For each phase assignment in all possible phase assignment set $A_{phase}$ generated based on $U_{phase}$, its permutation cost signature $C_p$ is computed by the pre-trained weighted function~\cite{zhou2020fast} with $S^0_{scc}(f)$ and $S^1_{scc}(f)$~(line 11 to line 13).
If the cost signature is the lowest of the phase assignments, refresh the candidate phase assignment set $C_{phase}$~(line 14 to line 16).
The cost-aware enumeration estimation method using $S_{scc}(f)$ is proposed by ~\cite{zhou2020fast}.
For each candidate phase assignment in $C_{phase}$, exhaustively enumerate the undetermined input order to get permutation transformations~(line 17 to line 18).
Thus, the output polarity, phase assignment, and input order of each variable have already been decided, and we can get the minimum truth table as the canonical form.

Here we give an example to explain the algorithm. For a function 
\begin{equation}
\nonumber
\begin{aligned}
    f &= x_1x_2+x_1(\overline{x_5+x_6(\overline{x_3}+x_4)}) +x_2(x_5+x_6(\overline{x_3}+x_4)) \\
    &=(\mathrm{FFFF3777C8880000})_{16},
\end{aligned}
\end{equation}
the $S_{cof}(f)$=$(|f|, |f_{x_1}|,...,|f_{x_6}|)$=$(32, 16, 16, 16, 16, 27, 21)$, and variables are separated into $G$=$\{(x_1,x_2,x_3,x_4),(x_5),(x_6)\}$~(line 3).
The $U_{phase}$=$\{x_1,x_2,x_3,x_4\}$ because $|f_{x_1}|$ to $|f_{x_4}|$ are half of $|f|$~(line 3).
The $U_{perm}$=$\{x_1,x_2,x_3,x_4\}$ since their cofactor signatures are the same~(line 4).
Next, we only need to detect symmetry in $G_1$=$\{x_1, x_2, x_3, x_4\}$.
Because variables $x_3$ and $x_4$ are symmetric, $x_4$ can be purged in $G_1$, $U_{perm}$ and $U_{phase}$~(line 5 to line 8), that is $G_1$=$U_{perm}$=$U_{phase}$=$\{x_1, x_2, x_3\}$.
After these steps, the total number of residual enumerations is pruned to only $2^3$ $\times$ $3!$=48.
Next, the $S_{inf}(f)$=$(6,10,2,2,10,22)$, so $G$ can be further separated into five groups, $\{(x_1),(x_2),(x_3),(x_5),(x_6)\}$~(line 9).
The $U_{perm}$=$\varnothing$ because permutation only appears in the same group~(line 10).
Since the influence signature is negation-independent, $U_{phase}$ does not change.
After the influence-aided permutation pruning, input order is determined, and the residual enumerations only consist of phase assignments, the number of which is $2^3$=8.
According to the cost estimation method described in~\cite{zhou2020fast}, phase assignment $\overline{x_1}\overline{x_2}x_3$ has the lowest permutation cost among all 8 phase enumerations~(line 12 to line 17).
For this candidate phase assignment, $U_{perm}=\varnothing$; thus, no exhaustive permutation is required~(line 18 to line 19).
Then the algorithm gets the truth table with the determined phase assignment and input order~(line 20).

Figure~\ref{fig:computeexample} depicts the pruning results.
It is a heat map of the remaining enumerations.
We divide the pruning stages into three parts, called pruning by $S_{cof}(f)$~(line 2 to line 8), pruning by $S_{inf}(f)$~(line 9 to line 10), and pruning by $S_{scc}(f)$~(line 12 to line 17).
After each pruning stage, we will increase its activity if an enumeration has not been pruned. 
The darkest points in the figure represent the enumerations with the highest activity, which means that these  are irreducible.
From this heat map, we can see more clearly that influence is a vital ingredient for pruning permutations.

\begin{figure}[tbp]
\setlength{\belowcaptionskip}{0pt}
\centering
\begin{subfigure}[b]{0.49\textwidth}
\centering
\includegraphics[scale=0.15]{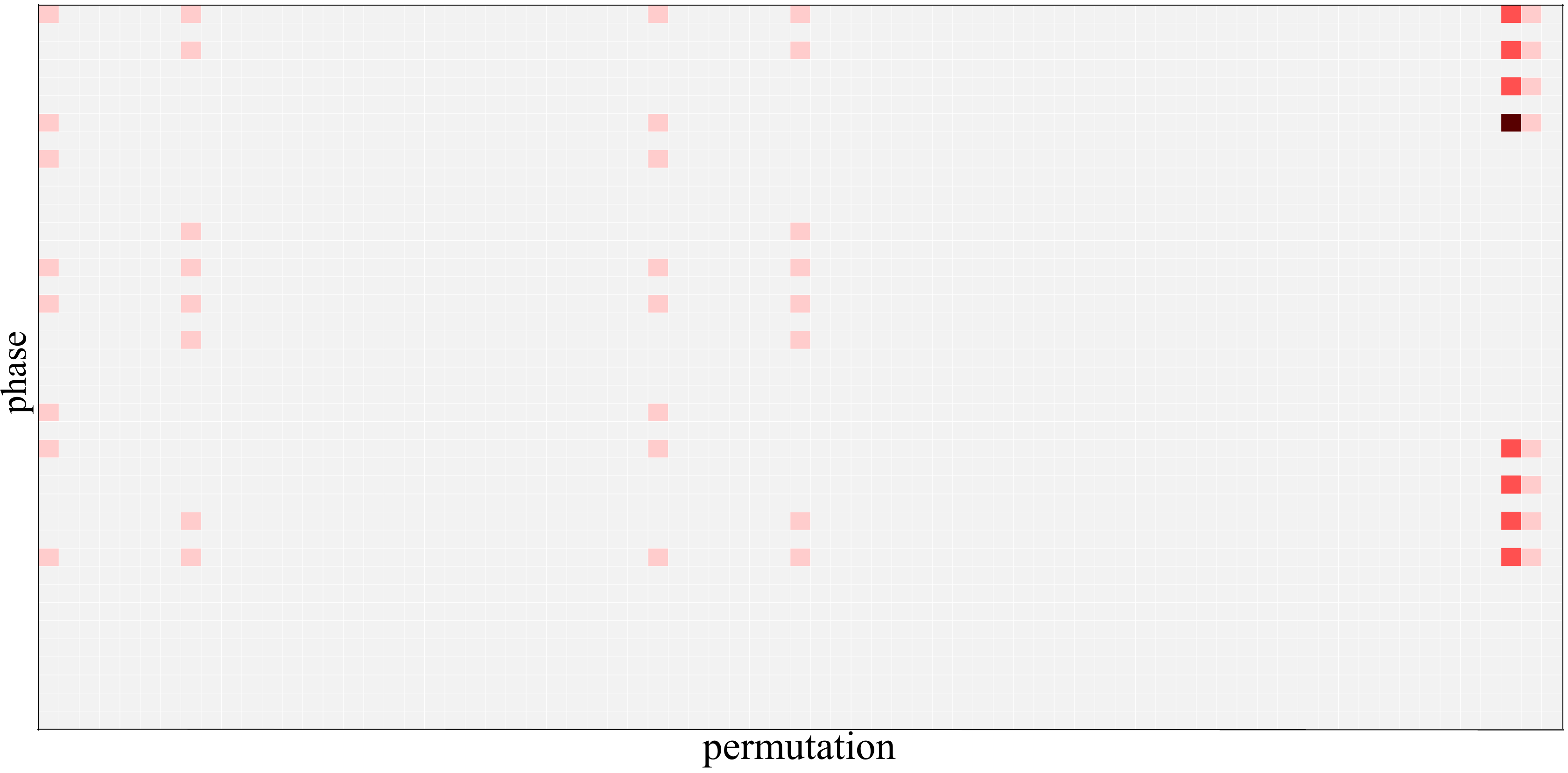}
\label{fig:1original}
\end{subfigure}
\caption{The heatmap of the remaining enumerations at different pruning stages.
The horizontal axis represents the permutation enumeration space, encoded according to the lexicographical order from 1 to n; the vertical axis represents the phase assignment enumeration space, according to binary encoding.
A point represents an NPN transformation, and all points constitute the whole NPN transformation space.}
\vspace{-0.8em}
\label{fig:computeexample}
\end{figure}

\section{Optimized Influence-aided Canonical Form}
\label{sec:optcanonical}

\subsection{Runtime Analysis and Trade-off}
Cost signature $C_{p}$ is proposed to estimate permutation cost and help choose the phase assignment with the lowest permutation cost~\cite{zhou2020fast}.
The permutation cost is computed according to $S^1_{scc}$, which will bring runtime overhead.
In Section~\ref{sec:infexample}, we have shown that influence signature usually gains more groups for further permutation reduction than shifted-cofactor signatures.
Although using influence signature and cost signature simultaneously can further prune the permutation space in some cases, the runtime overhead of computing $C_{p}$ will reduce the benefits.
In other words, using only influence signature to prune permutation space may gain shorter runtime.

\subsection{Optimized Influence-aided Canonical Form}

\begin{definition}
(\emph{optimized influence-aided signature vector}). The \emph{optimized influence-aided signature vector} of $f$ , denote by $S_{oi}(f)$, is a vector composed of cofactor signatures $S_{cof}(f)$, influence signatures $S_{inf}(f)$, the 0th-order shifted-cofactor signature $S^0_{scc}(f)$, followed by the truth table $T(f)$, that is,
\begin{equation}
\nonumber
\begin{aligned}
    S_{hi}(f) =&S_{cof}(f)  S_{inf}(f) C_{p}(f)  S^0_{scc}(f) T(f) \\
     =&\Big( \lvert f \rvert, |f_{x_{1}}|, |f_{x_{2}}|,..., |f_{x_{n}}|, \\  
      &inf_{x_1}(f),inf_{x_2}(f),...,inf_{x_n}(f),  C_{p}(f), \|f\|, \\
      &f(X_{(2^n-1)}), ..., f(X_{(1)}),  f(X_{(0)} ) \Big),
\end{aligned}
\end{equation}
where the hybrid signatures are concatenated as a larger vector.
\end{definition}

\begin{definition}
(\emph{optimized influence-aided canonical form}). 
Similarly, we can define \emph{optimized influence-aided canonical form} $\kappa_{oi}(f)$ as the function in $[f]$ with the minimum $S_{oi}(f)$, i.e. $\kappa_{oi}(f)=min_{S_{oi}}([f])$, where $[f]$ represents the NPN equivalent class of $f$.
\end{definition}

\begin{table*}[t]
\vspace{-0.6em}
\caption{Comparison of exact NPN classification methods.}
\label{tab:comparewithsota}
\centering
\small
\begin{tabular}{IcIcIcIcIcIcIcIcIcIcI}
\shline
 & \multirow{2}{*}{n}   & \multirow{2}{*}{\#Funcs} & \multirow{2}{*}{\makecell[c]{\#Classes}}  & \multicolumn{3}{cI}{Runtime~(s)} & \multicolumn{3}{cI}{Speedup} \\ 
 \cline{5-10}
 &   &    &  &	SOTA~\cite{zhou2020fast} & \textbf{INF}& \textbf{INF+} & SOTA/\textbf{INF}  &	SOTA/\textbf{INF+}	&  \textbf{INF}/\textbf{INF+} \\
\shline
\multirow{12}{*}{\emph{EPFL}} 
& 5  & 678376  & 370    & 0.142  & 0.136  & 0.138  & \textbf{1.04} & 1.03 & 0.99  \\
& 6  & 1054270 & 2339   & 0.579  & 0.459  & 0.471   & \textbf{1.26} & 1.23 & 0.98  \\
& 7  & 730338  & 8824   & 1.542  & 1.037  & 1.014  & 1.49 & \textbf{1.52} & 1.02  \\
& 8  & 1028501 & 27779  & 5.285  & 3.248  & 2.975  & 1.63 & \textbf{1.78} & 1.09  \\
& 9  & 177974  & 26731  & 3.123  & 2.397  & 1.952   & 1.30 & \textbf{1.60} & 1.23  \\
& 10 & 193905  & 50409  & 8.782  & 6.854  & 5.188  & 1.28 & \textbf{1.69} & 1.32  \\
& 11 & 207345  & 80282  & 24.047  & 18.816  & 13.683  & 1.28 & \textbf{1.76} & 1.38  \\
& 12 & 159741  & 87272  & 50.429  & 37.800  & 26.631  & 1.33 & \textbf{1.89} & 1.42  \\
& 13 & 107069  & 74111  & 97.417  & 68.540  & 47.591  & 1.42 & \textbf{2.05} & 1.44  \\
& 14 & 117825  & 85911  & 293.753  & 185.671  & 135.869  & 1.58 & \textbf{2.16} & 1.37  \\
& 15 & 123699  & 94106  & 1142.931  & 513.366  & 385.223  & 2.23 & \textbf{2.97} & 1.33  \\
& 16 & 127102  & 97283  & 5979.359  & 1378.776  & 1082.368  & 4.34 & \textbf{5.52} & 1.27  \\
\shline
\multirow{12}{*}{\emph{MCNC}}
& 5  & 1137438 & 398	& 0.219	& 0.211  & 0.215  & \textbf{1.04} & 1.02 & 0.98 \\
& 6  & 2461375 & 2558	& 0.969	& 0.820  & 0.830   & \textbf{1.18} & 1.17 & 0.99 \\
& 7  & 816538  & 8884	& 1.816	& 1.146  & 1.138  & 1.58 & \textbf{1.60} & 1.01 \\
& 8  & 1171120 & 28031  & 6.402	& 3.680  & 3.394   & 1.74 & \textbf{1.89} & 1.08 \\
& 9  & 92188   & 18394	& 2.769	& 1.980  & 1.578  & 1.40 & \textbf{1.75} & 1.25 \\
& 10 & 113205  & 32796	& 9.366	& 6.870  & 5.334  & 1.36 & \textbf{1.76} & 1.29 \\
& 11 & 96452   & 32153	& 17.997	& 13.859  & 10.299  & 1.30 & \textbf{1.75} & 1.35 \\
& 12 & 125937  & 45546	& 61.001	& 46.436  & 34.720  & 1.31 & \textbf{1.76} & 1.34 \\
& 13 & 137135  & 58921	& 230.835	& 148.387  & 116.286  & 1.56 & \textbf{1.99} & 1.28 \\
& 14 & 143417  & 72602	& 843.348	& 479.178  & 384.648  & 1.76 & \textbf{2.19 }& 1.25 \\
& 15 & 167385  & 85417	& 3154.421	& 1458.433  & 1258.687  & 2.16 & \textbf{2.51} & 1.16 \\
& 16 & 124425  & 47035	& 5485.760	& 1731.770  & 1679.458  & 3.17 & \textbf{3.27} & 1.03 \\
\shline
\end{tabular}
\end{table*}

Similar to Theorem ~\ref{tm:hybridinfvec} and Theorem ~\ref{tm:hybridinfcf}, we can prove that the uniqueness of $S_{oi}(f)$, and $\kappa_{oi}(f)$ is an NPN canonical form.
The computation of $\kappa_{oi}(f)$ is also similar to the computation method of $\kappa_{hi}(f)$ shown in algorithm~\ref{algo:compute}, except that the part of computing $S^1_{scc}(f)$ is correspondingly deleted.
Candidate phase assignments are selected through the single value $S^0_{scc}(f)$.

\section{Evaluation}\label{sec:experiments}
\subsection{Experiment Setup}

We implement the new canonical form and computation method based on Berkeley \textit{ABC}~\cite{abc}.
The state-of-the-art NPN classification method~\cite{zhou2020fast} is also integrated into \textit{ABC} as command \textit{testnpn -A 11}.
The procedure runs on an Intel Xeon Gold 6252 CPU 24-core computer with 128GB RAM.

We use MCNC~\cite{yang1991logic} and EPFL benchmarks~\cite{amaru2015epfl} to test the effectiveness of our algorithm on real synthesis applications.
We enumerate the $K$-cuts in the circuits of the benchmark suites to obtain the $K$-input Boolean functions and extract the truth tables.

\subsection{Comparison with State-of-the-Art Classification Method}

Table~\ref{tab:comparewithsota} compares the influence-aided canonical form method~(\textbf{INF}/\textbf{INF+}) with the state-of-the-art~(SOTA)~\cite{zhou2020fast} method.
Among them, \textbf{INF} corresponds to the hybrid influence-aided canonical form described in Section~\ref{sec:canonical}, and \textbf{INF+} corresponds to the optimized influence-aided canonical form described in Section~\ref{sec:optcanonical}. 
All three methods can achieve exact NPN classification.

Table~\ref{tab:comparewithsota} shows the number of functions~(\#Funcs) with different input bits~(n), the number of exact equivalence classes~(\#Classes), the overall running time of the three methods and the speedup ratio between them.
Both of our two influence-aided canonical form methods run faster than the SOTA for all NPN classification tasks.
Among them, the hybrid influence-aided canonical form method gains up to 4.34x speedup than the SOTA while the optimized influence-aided canonical form method achieves up to 5.52x speedup.
Notably, our influence-aided canonical form methods achieve better speedup as input increases.
Compared to \textbf{INF}, \textbf{INF+} achieves better speedup when the input bit is greater than 6. 
This further demonstrates that influence has superior pruning efficacy over shifted-cofactor.
Detailed analysis will be conducted in the next subsection.

\subsection{Discussion}

This subsection explains why our influence-aided method can run faster by showing some intermediate results. 

\begin{table*}[htbp]
\caption{Comparison between the size of variable groups for SODA and \textbf{INF}/\textbf{INF+} on EPFL benchmark.}
\vspace{-0.6em}
\label{tab:group}
\centering
\begin{tabular}{IcIccIccIccIccIccIccIccIccI}
\shline
\multirow{2}{*}{n} & \multicolumn{2}{cI}{\textbf{\#Total}} & \multicolumn{2}{cI}{\textbf{var-[1]}(\%)}            & \multicolumn{2}{cI}{\textbf{var-[2]}(\%)}  & \multicolumn{2}{cI}{\textbf{var-[3]}(\%)}   & \multicolumn{2}{cI}{\textbf{var-[4,5]}(\%)} & \multicolumn{2}{cI}{\textbf{var-[6,7,8]}(\%)} & \multicolumn{2}{cI}{\textbf{var-[$\geq$9]}(\%)} \\ \cline{2-15}
& \textbf{SODA} & \textbf{INF}  & \textbf{SODA}  & \textbf{INF} & \textbf{SODA} & \textbf{INF} & \textbf{SODA} & \textbf{INF}  & \textbf{SODA} & \textbf{INF} & \textbf{SODA} & \textbf{INF} & \textbf{SODA} & \textbf{INF} \\ \shline
5  & 2675  & 2230    & 54.505 & 84.619 & 25.944 & 13.543 & 10.953 & 1.839 & 8.598  & 0     & 0.112 & 0     & 0      & 0    \\
6  & 21532 & 17658   & 56.154 & 86.023 & 17.885 & 11.400 & 11.453 & 1.597 & 14.509 & 0.980 & 4.575 & 0     & 0      & 0     \\
8  & 235466 & 245195 & 69.375 & 89.617 & 14.313 & 8.481  & 5.553  & 1.518 & 8.971  & 0.352 & 3.770 & 0.027 & 0      & 0     \\
10 & 334259 & 359487 & 85.115 & 95.474 & 10.627 & 4.299  & 2.040  & 0.176 & 1.560  & 0.038 & 0.633 & 0.003 & 0      & 0     \\
12 & 619120 & 679659 & 88.739 & 96.451 & 8.255  & 3.369  & 1.428  & 0.138 & 0.843  & 0.041 & 0.290 & 0.003 & 0.0305 & 0     \\
13 & 556337 & 616669 & 90.002 & 96.974 & 7.238  & 2.891  & 1.283  & 0.121 & 0.715  & 0.014 & 0.298 & 0.001 & 0.0521 & 0     \\
15 & 810309 & 904187 & 90.778 & 97.171 & 6.654  & 2.674  & 1.226  & 0.128 & 0.520  & 0.023 & 0.172 & 0.003 & 0.2968 & 0     \\
16 & 884091 & 986068 & 91.165 & 97.206 & 6.400  & 2.625  & 1.168  & 0.138 & 0.497  & 0.028 & 0.159 & 0.006 & 0.3595 & 0     \\ 
\shline
\end{tabular}
\end{table*}

\begin{table*}[htbp]
\vspace{-0.6em}
\caption{Details of \#phase, \#permutation, final exhaustive enumeration, and runtime of Cp on EPFL benchmark.}
\label{tab:timecompare}
\centering
\begin{tabular}{IcIcccIcccIcccIcccI}
\shline
\multirow{2}{*}{n} & \multicolumn{3}{cI}{\textbf{\#Phase}} & \multicolumn{3}{cI}{\textbf{\#Perm}} & \multicolumn{3}{cI}{\textbf{\#Enum}} & \multicolumn{3}{cI}{\textbf{Time$_{\textbf{cp}}$(sec)}} \\
\cline{2-13}
& \textbf{SODA} & \textbf{INF}    & \textbf{INF+} & \textbf{SODA} & \textbf{INF}    & \textbf{INF+} & \textbf{SODA} & \textbf{INF}    & \textbf{INF+} & \textbf{SODA} & \textbf{INF}    & \textbf{INF+}   \\ \shline
5   & 1149  & 555   & 649    & 3541      & 1740   & 2703      & 1238   & 712    & 712    & 0.006914   & 0.002327   & 0.001208    \\
6   & 10110 & 4296  & 5471   & 31678     & 13304  & 39425     & 9111   & 4727   & 4727   & 0.039637   & 0.024064   & 0.014824    \\
8   & 66800 & 44128 & 73050  & 354438    & 162677 & 2673064   & 69353  & 49049  & 49049  & 0.409565   & 0.937477   & 0.588323    \\
10  & 27317 & 24705 & 41901  & 291860    & 167821 & 345734    & 63186  & 57773  & 57773  & 1.646048   & 2.157659   & 0.811257    \\
12  & 33884 & 31871 & 79205  & 988519    & 222595 & 422534    & 95719  & 92352  & 92352  & 6.400379   & 14.529945  & 5.719692    \\
13  & 26620 & 25324 & 81517  & 2290221   & 181371 & 371319    & 78696  & 77051  & 77051  & 21.638766  & 28.091138  & 12.774357   \\
15  & 34212 & 33170 & 212008 & 28520852  & 249233 & 19550823  & 98681  & 97393  & 97393  & 230.843127 & 225.61318  & 151.732984  \\
16  & 35280 & 34417 & 322081 & 110090650 & 263786 & 37806293  & 101292 & 100238 & 100238 & 866.437428 & 615.558448 & 463.198676  \\
\shline
\end{tabular}
\end{table*}

Section~\ref{sec:infexample} briefly mentioned that introducing influence can increase the number of groups, thereby reducing the number of variables within the same group. 
From the introduction of the canonical form method in Section~\ref{sec:canonical}, it is evident that the more variables there are within the same group, the greater the number of enumerations required subsequently. 
Table~\ref{tab:group} presents the total number of groups and the statistics of different variable groupings for both the SOTA and the INF/INF+ methods. 
The notation var-[x] represents the number of groups with x variables, expressed as a percentage of the total number of groups.
This table shows that incorporating INF features for variable grouping increases the number of groups when the input is larger than five while simultaneously reducing the number of variables within each group. 
This also implies that \textbf{INF}/\textbf{INF+} will prune more transformations.

\begin{figure}[h]
\centering
\includegraphics[scale=0.25]{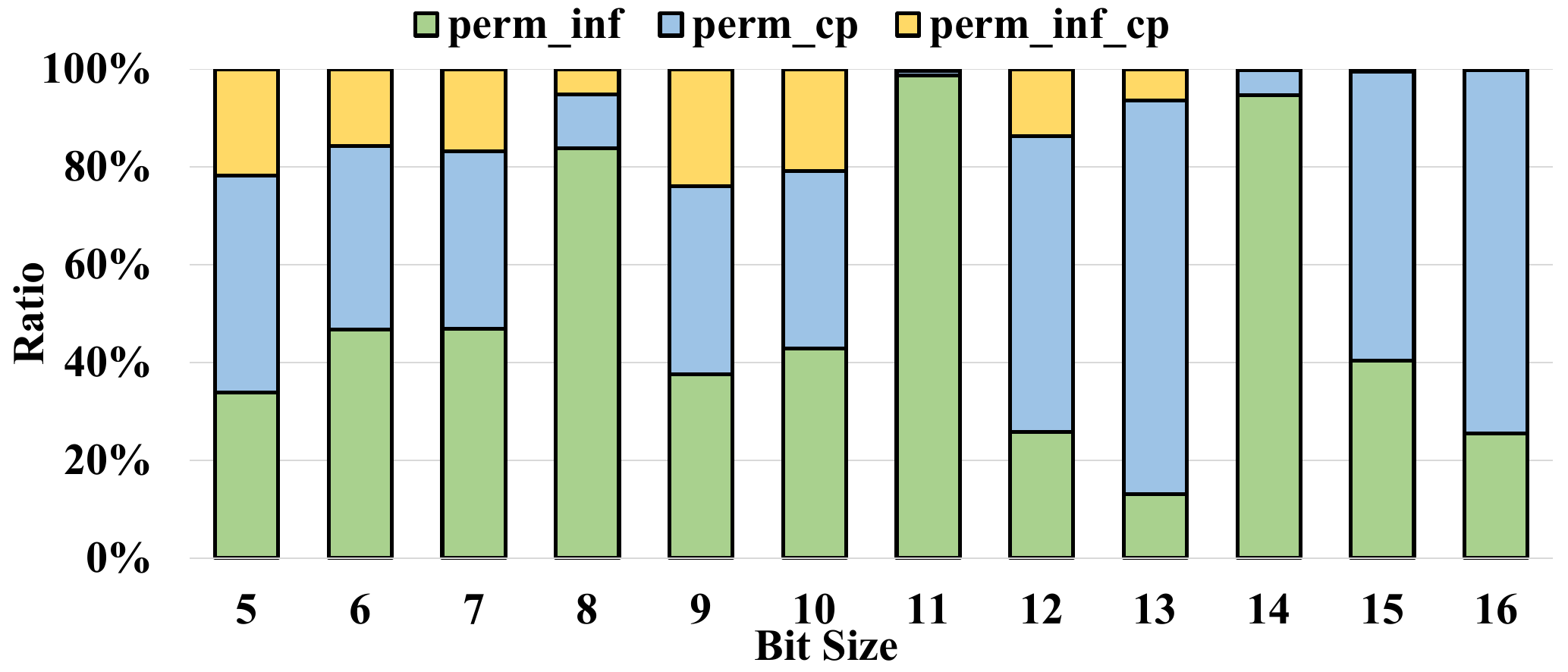}
\caption{\#Permutation enumerations of different methods.}
\label{fig:enumeration}
\end{figure}

Benefiting from the permutation-dependent property of influence, we can use influence with cofactor and cost signature to further reduce the permutation space. 
We collected the number of remaining permutation enumerations separately pruned by $S_{inf}$, $C_{p}$, and $S_{inf}$+$C_{p}$.
Fig.~\ref{fig:enumeration} shows the results. 
The y-axis includes the proportion of \#permutation of each among the three above.
The total permutation enumerations can be reduced significantly when incorporating influence signatures.
Thus, the \textbf{INF} and \textbf{INF+} can gain speedup compared with the SOTA.


Table~\ref{tab:timecompare} provides a detailed summary of the EPFL benchmark, showcasing the total number of transformations and the cumulative runtime for the three canonical form methods. 
The headings of the columns are as follows: \#Perm denotes the remaining permutation transformations after symmetry detection;
\#Phase represents the remaining phase assignment transformations after symmetry detection;
\#Enum indicates the total number of final exhaustive transformations after all pruning steps, 
and Time$_{\textbf{cp}}$ denotes the runtime for cost signature computation.
This table shows that both \textbf{INF} and \textbf{INF+} can realize less transformation for phase, permutation, and final exhaustive enumeration compared to SOTA.
And it is also evident that computing $C_{p}$ needs a long time.
Even though it could further prune the permutation enumeration,
the time overhead for computing $C_{p}$ outweighs the enumeration reduction.
Thus, \textbf{INF+} gains a better speedup than \textbf{INF} when $n$ is larger.
Considering the trade-off in selecting pruning signatures is essential when developing a fast NPN classification method.

\section{Conclusion and Future Work}\label{sec:conclusion}
This paper describes a novel canonical form and its computation algorithm with consideration of Boolean influence.
We highlight that influence signature is negation independent and permutation dependent and design two influence-aided canonical form methods for exact NPN classification.
Experimental results prove that influence plays a significant role in reducing the transformation enumeration in computing the canonical form, and our influence-aided canonical forms gain up to 5.5x speedup than the state-of-the-art NPN classification method.
In the future, we will try different combinations of signatures in the canonical form and attempt to apply similar methods to other applications like NPNP matching.

\section*{Acknowledgments}
This work is partly supported by 
the Major Key Project of PCL (No. PCL2023AS2-3),
the National Natural Science Foundation of China (No. 62090021), 
the National Key R\&D Program of China (No.2022YFB4500500) and (No. 2022YFB4500403), 
the Strategic Priority Research Program of Chinese Academy of Sciences (No. XDA0320300),
the Ministry of Education of China (No. 20YJA880001) and the Innovation Program for Quantum Science and Technology (No. 2021ZD0302900).

\newpage
\bibliographystyle{IEEEtran}
\bibliography{dac23influence.bib}

\end{CJK*}
\end{document}